\documentclass{article}

\PassOptionsToPackage{numbers, sort&compress}{natbib}
\usepackage{arxiv}
\usepackage{amsmath}
\usepackage{amsthm}

\usepackage[utf8]{inputenc} % allow utf-8 input
\usepackage[T1]{fontenc}    % use 8-bit T1 fonts
\usepackage{url}            % simple URL typesetting
\usepackage{booktabs}       % professional-quality tables
\usepackage{amsfonts}       % blackboard math symbols
\usepackage{nicefrac}       % compact symbols for 1/2, etc.
\usepackage{xcolor}         % colors
\usepackage{microtype}
\usepackage{graphicx}
\usepackage{subcaption}
\usepackage{booktabs} % for professional tables
\usepackage{xcolor}

\usepackage{pgfplots}
\pgfplotsset{compat = newest}
\usepackage{multirow}

\usepackage{framed}
\usepackage{tcolorbox}

\usepackage{amsmath}
\usepackage{amssymb}
\usepackage{mathrsfs}
\usepackage{mathtools}
\usepackage{amsthm}
\usepackage[ruled,vlined]{algorithm2e}
\usepackage{algorithmic}
\usepackage{listings}
\usepackage{makecell}
\usepackage{colortbl}
\usepackage{color}
\usepackage{wrapfig}
\usepackage{cancel}
\usepackage{soul,xcolor}
\usepackage{pifont}
\usepackage{tikz}
\usetikzlibrary{shapes, positioning, arrows.meta, calc, decorations.pathmorphing, quotes}

\usepackage[capitalize,noabbrev]{cleveref}

\theoremstyle{plain}
\newtheorem{theorem}{Theorem}[section]

\newtheorem{lemma}[theorem]{Lemma}

\theoremstyle{definition}
\newtheorem{definition}[theorem]{Definition}
\newtheorem{assumption}[theorem]{Assumption}

\theoremstyle{remark}

\newcommand{\alink}[1]{\href{#1}{paper-link}}

\usepackage{enumitem}
\usepackage{ amssymb }

\definecolor{citecolor}{HTML}{0071BC}
\definecolor{linkcolor}{HTML}{ED1C24}

\definecolor{commentcolor}{RGB}{110,154,155}   
\usepackage{amsmath,amsfonts,bm}

\def\eqref#1{equation~\ref{#1}}
\def\1{\bm{1}}

\DeclareMathAlphabet{\mathsfit}{\encodingdefault}{\sfdefault}{m}{sl}
\SetMathAlphabet{\mathsfit}{bold}{\encodingdefault}{\sfdefault}{bx}{n}

\title{The Information of Large Language Model Geometry}
\author{%
    \textbf{Zhiquan Tan$^1$ 
    \quad
    Chenghai Li$^2$
    \quad
    Weiran Huang$^{3}$}\\[0.3cm]
    $^1$ Department of Mathematical Sciences, Tsinghua University\\
    $^2$ Independent\\
    $^3$ Qing Yuan Research Institute, SEIEE, Shanghai Jiao Tong University
}

\begin{document}

\maketitle

%%%%%%%%%%%%%%%%%%%%%%%%%%%%%%
%%% Abstract

\begin{abstract}
This paper investigates the information encoded in the embeddings of large language models (LLMs). We conduct simulations to analyze the representation entropy and discover a power law relationship with model sizes. Building upon this observation, we propose a theory based on (conditional) entropy to elucidate the scaling law phenomenon. Furthermore, we delve into the auto-regressive structure of LLMs and examine the relationship between the last token and previous context tokens using information theory and regression techniques. Specifically, we establish a theoretical connection between the information gain of new tokens and ridge regression. Additionally, we explore the effectiveness of Lasso regression in selecting meaningful tokens, which sometimes outperforms the closely related attention weights. Finally, we conduct controlled experiments, and find that information is distributed across tokens, rather than being concentrated in specific "meaningful" tokens alone.
\end{abstract}

%%%%%%%%%%%%%%%%%%%%%%%%%%%%%%
%%% Introduction
\section{Introduction}

Large language models (LLMs) have brought about a revolution in natural language processing, enabling significant breakthroughs across various tasks \citep{brown2020language, touvron2023llama}. However, comprehending the underlying mechanisms that drive their impressive performance remains a challenging endeavor. One interesting phenomenon observed in LLMs is the scaling law \citep{kaplan2020scaling, hoffmann2022training}, which shows that the performance of models follows a predictable power-law relationship with quantities like model parameters. While scaling laws have been empirically discovered in diverse domains, the theoretical understanding of it is rare \citep{bahri2021explaining, michaud2023quantization}. The understanding of scaling laws in LLMs is of great importance as it can provide insights into the scalability, efficiency, and generalization capabilities of these models as they grow larger.

To begin unraveling the mysteries of scaling laws in LLMs, we conduct simulations \citep{wei2024large} to analyze the representation entropy, revealing an interesting finding: a power law relationship may exist between representation entropy and model sizes. This discovery shows that LLM representation (geometry) has a deep connection with information theory. Building upon this insight, we propose a novel theoretical framework based on (conditional) entropy to provide a deeper understanding of the scaling law phenomenon. Our theory offers valuable insights into how information is encoded and utilized within LLMs as they scale.

When dealing with scaling law, we consider each sentence as a whole, therefore focusing more on a ``macro'' behavior. When considering the ``micro'' structure, we shift our focus to the auto-regressive structure inherent in LLMs, specifically examining the relationship between the last token and the preceding context tokens. Our analysis establishes a compelling theoretical connection between the information gain brought by new tokens and ridge regression. This connection not only deepens our understanding of the auto-regressive nature of LLMs but also provides insights into how the addition of new tokens contributes to the overall information content of the models. Another critical aspect we explore is the effectiveness of token selection mechanisms within LLMs. Surprisingly, our analysis reveals that Lasso regression sometimes outperforms the closely related attention weights in identifying context tokens with high information content. 

Finally, we investigate whether a token embedding contains all the information from its preceding context. By comparing the result given by mean embedding and specific token embedding, we find that information is encoded among all tokens, not concentrated on specific tokens. Additionally, we design metrics that capture sentence-level distance and find it also successfully distinguishes sentences of different meanings. These empirical investigations further validate that the ``meaning'' of a sentence should contain all tokens' information, even in the auto-regressive setting.

\section{Background}

\subsection{Information-theoretic Quantities}

In this section, we will present the standard information quantities used in this paper. For more detailed discussions on these quantities, please refer to the standard textbook in information theory \citep{coverelements}. We will present the definitions under the discrete setting for notation simplicity. For the continuous cases (differential entropy), one can change the $\sum$ to $\int$ in all definitions.

\begin{definition}[Entropy] \label{entropy}
Suppose a random variable $X$ has support on a set $\mathcal{X}$, then the entropy of $X$ is defined as:
\begin{equation*}
\operatorname{H}(X) = - \sum_{x \in \mathcal{X}} p(x) \log p(x)  .
\end{equation*}
\end{definition}

\begin{definition}[Conditional entropy]
Suppose random variables $X$ and $Y$ have supports on $\mathcal{X}$ and $\mathcal{Y}$, then the (conditional) entropy of $X$ conditioned on $Y$ is defined as:
\begin{equation*}
\operatorname{H}(X|Y) = \sum_{y \in \mathcal{Y}} p(y) \operatorname{H}(X|Y=y) = -\sum_{y \in \mathcal{Y}} p(y) \sum_{x \in \mathcal{X}} p(x|y) \log p(x|y),
\end{equation*}
where $\operatorname{H}(X|Y=y)$ resembles definition \ref{entropy}.
\end{definition}

\textbf{Remark}: From standard textbook \citep{coverelements}, one knows that conditional entropy is closely related to the approximation error (or Fano inequality in the discrete case), making it a good estimate of the ``global'' effect of $Y$'s influence on $X$.

\begin{definition}[KL divergence]
Suppose two distributions $P$ and $Q$ has support on a set $\mathcal{X}$, then KL divergence between $P$ and $Q$ is defined as:
\begin{equation*}
\operatorname{KL}(P \| Q) = - \sum_{x \in \mathcal{X}} P(x) \log \frac{Q(x)}{P(x)}  .
\end{equation*}
\end{definition}

\subsection{Neural scaling law}

\citep{kaplan2020scaling} empirically shows that the test loss $\mathcal{L}$ of a model can be modeled using a power-law relationship with the number of training flops ($S$), (non-embedding) parameters ($N$) and the dataset size ($D$) when other factors do not impose constraint. 

This can be expressed as the following equation:
\begin{equation*}
\mathcal{L} = \mathcal{L}_0 + \lambda Y^{\gamma},
\end{equation*}
where $\mathcal{L}$ represents the cross-entropy loss, $\mathcal{L}_0 $ is an irreducible term that can be seen introduced by noise, $Y$ denotes either $N$, $D$ or $S$, and $\gamma$ is coefficient in the power-law relationship. The constant $\lambda$ indicates other factors that may influence the $\mathcal{L}$, such as the backbone and the optimizers.

\section{Entropy in LLM (geometry)}

Note that standard cross-entropy loss is approximately the entropy of language models on the data, we wonder what is the case of \emph{representation} entropy.

As entropy is hard to estimate in high-dimensional spaces, we will use the following rate-distortion estimate of entropy \citep{coverelements, ma2007segmentation}:

\begin{definition} \label{log det entropy}

\begin{equation*}
\operatorname{L}_{\epsilon}(\mathbf{Z}) = \frac{d+n}{2}\log \det (\mathbf{I}_d  + \frac{d}{\epsilon^2} \Sigma_{\mathbf{Z}} ),
\end{equation*}
where $\mathbf{Z} \in \mathbb{R}^{n \times d}$ has each of its row a token representation and $\Sigma_{\mathbf{Z}} = \frac{1}{n} \mathbf{Z}^{\top}\mathbf{Z}$ the covariance matrix.
\end{definition}

By noticing the fact that the above estimation depends on $d$ and models of different sizes usually have different hidden dimensions, we will normalize the entropy by its maximum to give a fair comparison among different models.

Assume each row of $\mathbf{Z}$ has a unit norm, then the maximal possible value of $\operatorname{L}_{\epsilon}(\mathbf{Z})$ will be $\frac{d+n}{2} d \log(1 + \frac{1}{\epsilon^2})$. Assume $\mathbf{Z}$ has $r$ equal-sized singular values ($r$-subspace). Then the normalized entropy can be calculated as follows:

\begin{equation} \label{normalized entropy}
\frac{\operatorname{L}_{\epsilon}(\mathbf{Z})}{\frac{d+n}{2} d \log(1 + \frac{1}{\epsilon^2})} = (\frac{r}{d} \log(1+ \frac{\frac{1}{\epsilon^2}}{\frac{r}{d}})) / \log(1 + \frac{1}{\epsilon^2}).
\end{equation}

Note from the data measured from matrix entropy (Von Neumann's, not the Shannon's we used in this paper) \citep{wei2024large}, we find $\frac{r}{d}$ has a good estimate. We simulate the (normalized) entropy according to equation (\ref{normalized entropy}) and the data in \citep{wei2024large} and get a plot in Figure \ref{fig: entropy}, where we take $\epsilon=0.1$. It follows a power-law relationship.

\begin{figure}[t] 
\centering 
\includegraphics[width=0.55\columnwidth]{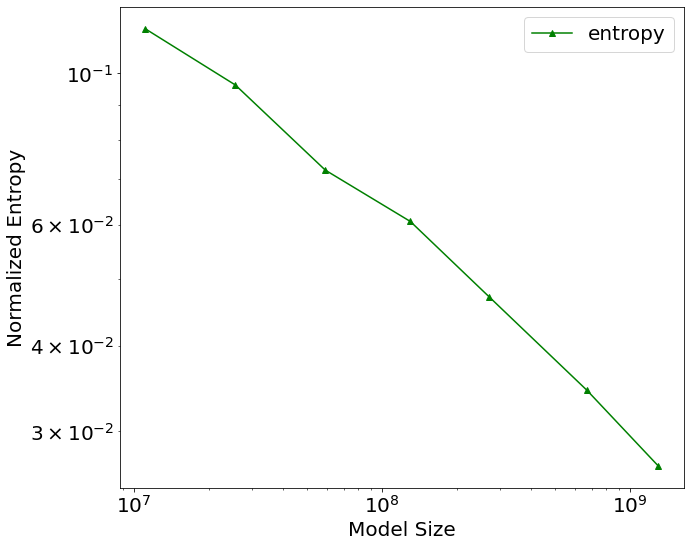}
\caption{The relationship of (normalized) entropy and model size.}
\label{fig: entropy}
\end{figure}

\textbf{Remark:} When considering the discrete entropy, as the vocabulary size is fixed, the normalization of the maximal entropy value will make no difference in the scaling law.

Having this empirical evidence, we will then discuss how to theoretically derive the scaling law from information theory. Assume we are given a large pretraining corpus $\mathcal{D}$. Note LLMs have a strong ability to generate complete sentences by extending sentence fragments. We then collect the sentences (paragraphs) in $\mathcal{D}$ as basic elements and denote their collections as $\mathcal{X}$.

\begin{figure}[t] 
\centering 
\includegraphics[width=0.8\columnwidth]{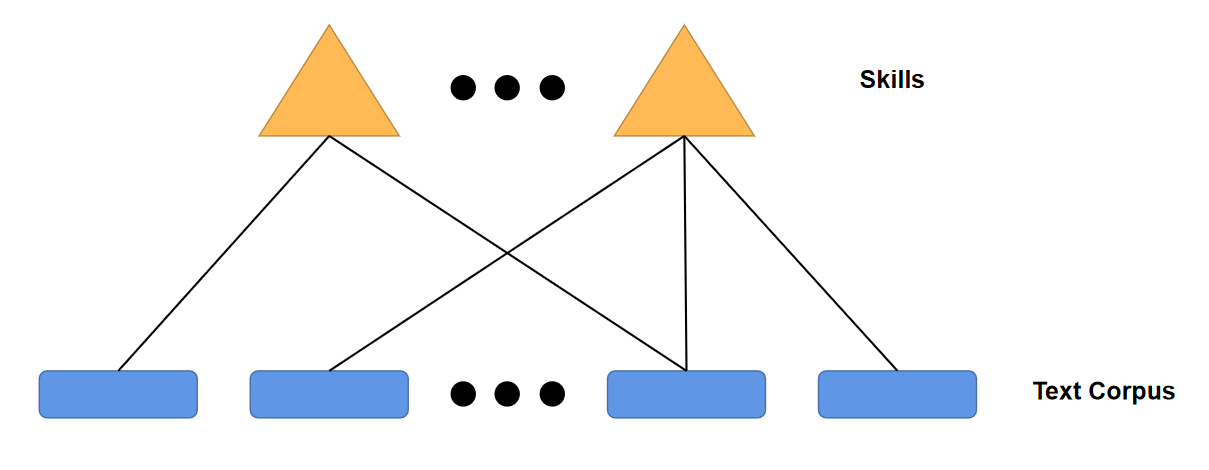}
\caption{A schematic view of corpus and skills.}
\label{fig: scheme}
\end{figure}

Assume there will be a total of $M$ skills possible to comprehend from any corpus. Following \citep{arora2023theory}, assume each text segment in the corpus will link to all its related skills and thus this will result in a bipartite graph which is called a skill graph. A conceptual picture of this construction can be seen in Figure \ref{fig: scheme}.

Assume that each skill $y_k$ has a degree proportional to $k^{-(\alpha + 1)}$, where we ranked skills by their degree in the bipartite graph. Therefore, the skill distribution follows a power law by normalizing the degrees, i.e. $p_{\text{skill}}(y_k) =  \frac{k^{-(\alpha + 1)}}{Z_{\alpha}}$ ($\alpha>0$), where $Z_{\alpha} = \sum^M_{k=1} k^{-(\alpha + 1)}$ is the normalization constant.

A natural question arises: As we want LLM to comprehend (learn) skills, what do comprehending skills mean? To answer this question, we will first discuss why entropy is naturally used for quantity ``comprehending'' by reviewing the basic knowledge in coding theory \citep{coverelements}. This will help us define what it means to ``comprehend'' skills defined above.

Intuitively, a description of a sentence will be a (binary) string and the length of it should be as short as possible if we actually ``comprehend'' it. We will formalize this in the following definitions.

\begin{definition}[Instantaneous code \citep{coverelements}]

A code $\mathcal{C}$ (in information theory) is a function that maps $\mathcal{X}$ into the space of binary strings (space of codewords), with none of the codeword $\mathcal{C}(x)$ be the prefix of another codeword $\mathcal{C}(x')$. 
\end{definition}

\begin{definition}[Expected code length \citep{coverelements}]
The expected code length of a code $\mathcal{C}$ is the average number of binary symbols under a distribution $\{p(x)\}_{x \in \mathcal{X}}$:
\begin{equation*}
l(\mathcal{C}) = \sum_{x \in \mathcal{X}} p(x) |\mathcal{C}(x) |.
\end{equation*}
\end{definition}

By the standard result of information theory \citep{coverelements}, \emph{any} expected code length is greater than $\operatorname{H}(X)$, where $X \sim \{p(x)\}_{x \in \mathcal{X}}$ (This even holds for the average length we use several concatenations of sentences in $\mathcal{X}$).

By far, we know that (conditional) entropy can be seen as a sort of ``descriptive'' length (code length) and by noticing that entropy has a close connection with (conditional) Kolmogorov complexity \citep{li2008introduction}. Then it is natural to have the following definition \ref{comprehend skill} from the discussion above.

\begin{definition}[Comprehend a skill]  \label{comprehend skill}
Assume the density function of $X \sim p_{\theta}$. As we are considering the behavior of LLM along the training, we assume initially (for a random initialized model) that all the skill conditional entropy is a constant $B$. Then for any skill $y$, a specific LLM $p_{\theta}$ comprehends this skill iff $\operatorname{H}(X|y) = C$, where $C < B$.
\end{definition}

\textbf{Remark:} This definition can also have an energy-based interpretation. Recalling the simplified capacity formula for the bandwidth-limited AWGN channel $\mathbf{C}_{\text{AWGN}} = W \log (1 + \frac{1}{W})$, where $W$ is the bandwidth (can be seen as energy). As entropy exceeding capacity can not be reliably transmitted, setting the entropy as $\mathbf{C}_{\text{AWGN}}$ and note it is an increasing function of $W$. We know that definition \ref{comprehend skill} means comprehending a skill is similar to having ``low energy''.

For all skills, we rank their ``hardness'' by their occurrence probability. That is, for any skills $y_1$ and $y_2$, we call skill $y_1$ is easier than $y_2$ iff $p_{\text{skill}}(y_1) \geq p_{\text{skill}}(y_2)$.
\begin{assumption}[Learn from easy to hard]
Assume the model learns skills from easiest to the hardest sequentially, only it completely comprehends a skill it will learn the next. 
\end{assumption}

We will show why the comprehension (emergence) \citep{wei2022emergent} of skills is closely linked with the scaling law of conditional entropy which will be discussed in theorems \ref{para scal} and \ref{flop scal}.

\begin{lemma} \label{skill lemma}
Assume a LLM comprehends $n$ skills and $X \sim p_{\theta}$, $Y \sim p_{\text{skill}}$. Then the conditional entropy $\operatorname{H}(X | Y)$ obeys a power law relationship with $n$.   
\end{lemma}

\begin{proof}
Note $M>>1$, we have
\begin{align*}
& \operatorname{H}(X|Y) \\
=& \sum_y p_{\text{skill}}(y) \operatorname{H}(X | y) \\
=& \sum^n_{i=1} p_{\text{skill}}(y_i) \operatorname{H}(X | y_i) + \sum_{i>n} p_{\text{skill}}(y_i) \operatorname{H}(X | y_i) \\
=& \sum^n_{i=1} p_{\text{skill}}(y_i) C + \sum_{i>n} p_{\text{skill}}(y_i) B \\
=& \sum^n_{i=1} p_{\text{skill}}(y_i) C + \sum_{i>n} p_{\text{skill}}(y_i) C - \sum_{i>n} p_{\text{skill}}(y_i) C + \sum_{i>n} p_{\text{skill}}(y_i) B \\
=& C + \sum_{i>n} p_{\text{skill}}(y_i) (B-C) \\ 
\sim & C + (B-C) \int^{M}_{y=n} \frac{y^{-\alpha - 1}}{Z_{\alpha}} dy  \\ 
\sim & O(n^{-\alpha}).
\end{align*}
\end{proof}

\begin{theorem}[Scaling law with parameter] \label{para scal}
Assume every $r$ neuron can comprehend a skill. Then $\operatorname{H}(X | Y) \sim O(N^{-\alpha}).$
\end{theorem}
\begin{proof}
This is evident from the fact that $n = \frac{N}{r}$ and use lemma \ref{skill lemma} for each \emph{optimized} model. Note the assumption here is similar to the ``quanta'' in \citep{michaud2023quantization}.
\end{proof}

\begin{theorem}[Scaling law with flops] \label{flop scal}
Assume when starting to comprehend each skill, each flop will reduce $\Delta$ conditional entropy. Then $\operatorname{H}(X | Y) \sim O(S^{-\frac{\alpha}{\alpha+2}}).$
\end{theorem}
\begin{proof}
Each flop will have a probability of $p_{\text{skill}}(y)$ to contribute learning skill $y$. Then skill $y_i$ will need $\frac{B-C}{ p_{\text{skill}}(y_i) \Delta}$ flops to comprehend. As the skills are learned sequentially, the number of skills comprehended by a \emph{optimized} model after $S$ flops for can be derived as follows: 
\begin{equation*}
S = \sum^n_{i=1} \frac{B-C}{\Delta} \frac{1}{p_{\text{skill}}(y_i)} \sim  \int^{n}_{y=1} y^{\alpha + 1} dy \sim O(n^{\alpha + 2}).
\end{equation*}

There the conclusion follows by combining lemma \ref{skill lemma}.
\end{proof}

We will then discuss the relationship between ``prompting'' and conditional entropy.

\begin{definition}[Perfect prompting function]
A function $f$ from skills to sentences is called a perfect prompting function iff it satisfies the Bayesian sufficiency \citep{bernardo2009bayesian} requirements as follows: 

Suppose $Y \sim p_{\text{skill}}$, for $\forall x,y$, $p_{\theta}(x |Y=y) = p_{\theta}(x |f(Y)= f(y))$.

\end{definition}

The following lemma shows that zero-shot prompting can be seen as a way of 'activating' the skills. 

\begin{lemma}[Skill ``activation''] \label{prompt and skill}
Given a perfect promoting function $f$ and suppose $Y \sim p_{\text{skill}}$, $\bar{Y}=f(Y)$. For any skill $y$, we have $$\operatorname{H}(X| Y=y) = \operatorname{H}(X| \bar{Y}=f(y)).$$     
\end{lemma}
\begin{proof}
This is evident by calling the definition of conditional entropy.
\end{proof}

Therefore, we can show that the conditional entropy of skills is equal to the conditional entropy of zero-shot prompting.
\begin{theorem}
Suppose $Y \sim p_{\text{skill}}$ and $f$ a perfect prompting function. Denote $\bar{Y}=f(Y)$ be the prompt variable whose distribution is induced by $f$ and $Y$. Then we will have the following:
\begin{equation*}
\operatorname{H}(X|Y) = \operatorname{H}(X | \bar{Y}).
\end{equation*}
\end{theorem}

\begin{proof}
Note $\operatorname{H}(X|Y) = \sum_y p(Y=y) H(X|Y=y) = \sum_y p(Y=y) H(X|\bar{Y}=f(y)) = \sum_{y'} p(\bar{Y} = y') H(X|\bar{Y}=y') = \operatorname{H}(X|\bar{Y})$. Note we utilize lemma \ref{prompt and skill} and the law of the unconscious statistician for help. 
\end{proof}

We will use the general notion of ``context'' to describe for example few-shot or in-context learning prompts \citep{brown2020language}, and chain-of-thought prompts \citep{wei2022chain}.

\begin{definition}
Assume $Y \sim p_{\text{skill}}$. A skill $y$ is called irrelevant to a context (sentences) iff the conditional independence relation holds $\forall x$ $p_{\theta}(x|\text{context}, Y=y) = p_{\theta}(x| Y=y)$.
\end{definition}

The relevance of context and skill is more subtle. Based on the intuition that if the context is relevant to a skill, then the conditional distribution when adding this context will be more ``focused'' (less ``uniform''). As entropy is natural to quantity ``uniformness''. We will have the following definitions.

\begin{definition}
Assume $Y \sim p_{\text{skill}}$. A context is called good to a skill $y$ iff $\operatorname{H}(X|\text{context}, Y=y) < \operatorname{H}(X| Y=y)$.

Similarly, a context is called bad to a skill $y$ iff $\operatorname{H}(X|\text{context}, Y=y) > \operatorname{H}(X| Y=y)$.
\end{definition}

\begin{lemma}
Assume $Y \sim p_{\text{skill}}$. If a skill $y$ is irrelevant to a context. Then $\operatorname{H}(X|\text{context}, Y=y) = \operatorname{H}(X| Y=y)$.
\end{lemma}

\begin{proof}
Note $\operatorname{H}(X|\text{context}, Y=y) = - \sum_x  p_{\theta}(x | \text{context}, Y=y) \log p_{\theta}(x | \text{context}, Y=y) = - \sum_x  p_{\theta}(x |  Y=y) \log p_{\theta}(x | Y=y) = \operatorname{H}(X|Y)$.
\end{proof}

Using the above lemmas, we can conclude the following relation of the effect of using context prompts.

\begin{theorem}[Influence of context prompts on performance]
Assume $Y \sim p_{\text{skill}}$ and the context will only be irrelevant or good (bad) for any skill, then prompting on this context will decrease (increase) conditional entropy.    
\end{theorem}
\begin{proof}
Note 
\begin{align*}
&\operatorname{H}(X|\text{context}, Y) \\
=& \sum_y p_{\text{skill}}(y)\operatorname{H}(X|\text{context}, Y=y) \\
=& \sum_{y \text{ irrelevant to context}} p_{\text{skill}}(y)\operatorname{H}(X|\text{context}, Y=y) + \sum_{y \text{ good to context}} p_{\text{skill}}(y)\operatorname{H}(X|\text{context}, Y=y) \\  
=& \sum_{y \text{ irrelevant to context}} p_{\text{skill}}(y)\operatorname{H}(X|Y=y) + \sum_{y \text{ good to context}} p_{\text{skill}}(y)\operatorname{H}(X|\text{context}, Y=y) \\ 
<& \sum_{y \text{ irrelevant to context}} p_{\text{skill}}(y)\operatorname{H}(X|Y=y) + \sum_{y \text{ good to context}} p_{\text{skill}}(y)\operatorname{H}(X|Y=y) \\ 
=&\operatorname{H}(X| Y).
\end{align*}

\end{proof}

We will then show why entropy also obeys a power law relationship.

\begin{assumption}[Invariant of distribution] \label{invariant}
Denote the probability distribution of LLM on $\mathcal{X}$ is $p_{\theta}$ and the conditional distribution of LLM for any skill $y$ is $p_{\theta}(\cdot | y)$. Assume the distribution difference $\operatorname{KL}(p_{\theta}(\cdot |y ) \| p_{\theta})$ is irrelevant to the LLM parameter $\theta$ during training.
\end{assumption}

\textbf{Remark:} The quantity $\operatorname{KL}(P \| Q)$ can be understood as the average number of extra bits (length) required to encode samples from distribution $P$ using a code optimized for distribution $Q$. This explanation makes it useful, as we are interested in the code ``length'' in the context of understanding LLM.

\begin{lemma} \label{KL and entro}
The difference between entropy and conditional entropy can be rewritten by the KL divergence. Specifically,
\begin{equation*}
 \operatorname{H}(X) - \operatorname{H}(X | Y) = \sum_y p_{Y}(y) \operatorname{KL}( p_{X}(\cdot |y) \| p_{X}). 
\end{equation*}
\end{lemma}

\begin{proof}
Note 
\begin{align*}
& \operatorname{H}(X) - \operatorname{H}(X | Y) \\
=& -\sum_{x} p_{X}(x) \log p_{X}(x) + \sum_{y} p_{Y}(y) \sum_x p(x | y) \log  p(x | y) \\
=&  -\sum_{x} p_{X}(x) [\log p_{X}(x) - \sum_{y} \frac{p(x,y)}{p_{X}(x)} \log  p(x | y)] \\
=& -\sum_{x} p_{X}(x) [\sum_y \frac{p(x,y)}{p_{X}(x)} \log p_{X}(x) - \sum_{y} \frac{p(x,y)}{p_{X}(x)} \log  p(x | y)] \\
=& -\sum_{x} p_{X}(x) \sum_y p(y | x) \log \frac{p_{X}(x)}{ p_{X}(x |y)} \\
=& -\sum_{x} \sum_y p_{Y}(y)p(x |y) \log \frac{p_{X}(x)}{ p(x |y)} \\
=& -\sum_{y} p_{Y}(y) \sum_y p(x |y) \log \frac{p_{X}(x)}{ p(x |y)} \\
=& \sum_{y} p_{Y}(y) \operatorname{KL}(p_{X}(\cdot |y ) \| p_{X}).
\end{align*}
\end{proof}

\begin{theorem}
Under assumption \ref{invariant}. Assume $X \sim p_{\theta}$ and $Y \sim p_{\text{skill}}$, then $\operatorname{H}(X) = \operatorname{H}(X | Y) + C$, where $C$ is a constant irrelevant to the LLM parameter $\theta$ during training.
\end{theorem}
\begin{proof}
Note by using the above lemma \ref{KL and entro}, we find
\begin{align*}
\operatorname{H}(X) - \operatorname{H}(X | Y) 
= \sum_{y} p_{\text{skill}}(y) \operatorname{KL}(p_{\theta}(\cdot |y ) \| p_{\theta}).
\end{align*}
By combining assumption \ref{invariant}, the conclusion follows.
\end{proof}

\textbf{Remark:} When dealing with representation, the entropy notion is the (continuous) entropy. Note the above conclusions also hold with the (discrete) entropy, which is in terms of the (pretraining) loss in LLM. It is widely known that LLM losses have scaling law. Interestingly, our theory may show that one direct way of improving a model's understanding of specific skills is to design better data that ``make this skill occur more often''. Also, the mixture of experts (MOE) architecture may be intuitively understood as different experts may be good at different skills.

\subsection{Scaling law for dataset size}

Given a large enough model $p_{\theta}$ (which is optimized for a dataset of size $D$) and unlimited compute, by definition \ref{comprehend skill} then all the conditional entropy $\operatorname{H}(X |Y=y)$ will be minimized to $C$. By noticing that in context learning can be seen as performing a step of gradient descent \citep{von2023transformers}, we assume $p_{\theta}(x|y) \sim p_{\theta + \Delta \theta}(x)$, where $\Delta \theta$ depends on $y$ and we will omit the dependency on notation for simplicity.

Recall the standard techniques from information geometry \citep{amari2016information}, as $\Delta \theta$ is small and $\operatorname{KL}(p_{\theta } \| p_{\theta}) = 0$, then 
\begin{align*}
\operatorname{KL}(p_{\theta + \Delta \theta} \| p_{\theta}) \approx  \operatorname{KL}(p_{\theta} \| p_{\theta + \Delta \theta})  \sim \frac{1}{2}  \Delta \theta^{\top} \mathbf{F} \Delta \theta,
\end{align*}
where $\mathbf{F} = \mathbb{E}_{x \sim p_{\theta}} [(\nabla_{\theta} \log p_{\theta}(x))(\nabla_{\theta} \log p_{\theta}(x))^{\top} ]$ is the Fisher information matrix.

From Taylor expansion, one can also obtain \citep{amari2016information}
\begin{equation*}
\mathbb{E}_{x \sim p_{\theta}} \| \log p_{\theta + \Delta \theta}(x) - \log p_{\theta}(x)  \|^2 \sim \frac{1}{2}  \Delta \theta^{\top} \mathbf{F} \Delta \theta.  
\end{equation*}

Therefore, we know that $\operatorname{KL}(p_{\theta}(\cdot |y ) \| p_{\theta}) \sim \mathbb{E}_{x \sim p_{\theta}} \| \log p_{\theta}(x |y ) - \log p_{\theta}(x)  \|^2$, as skill $y$'s relevant sentences $x$ is proportional to the number of ``skill tokens'' $D p_{\text{skill}}(y)$, then we can assume that $\mathbb{E}_{x \sim p_{\theta}} \| \log p_{\theta}(x |y ) - \log p_{\theta}(x)  \|^2 \sim A (D p_{\text{skill}}(y))$, where $A$ is a constant. 

As $Z_{\alpha}$ is the ``area'' of skills, therefore we assume it is proportional to $D^{\gamma}$. Using lemma \ref{KL and entro}, note $Z_{\alpha} \approx \frac{1}{\alpha}$ and $\alpha <<1$. We will have
\begin{align*}
\operatorname{H}(X) =&  \operatorname{H}(X | Y) +
 \sum_{y} p_{\text{skill}}(y) \operatorname{KL}(p_{\theta}(\cdot |y ) \| p_{\theta}) \\
=&  C + \sum_{i} p_{\text{skill}}(y_i) (ADp_{\text{skill}}(y_i)) \\
\sim & C + AD \frac{\alpha^2}{2 \alpha +1} \\
\sim & C + AD \alpha^2 \\
\sim & O(D^{1-2\gamma}). 
\end{align*}

Thus we have obtained the scaling law on dataset.

\section{The information in the auto-regressive process}

\subsection{Information gain and ridge regression}

As the information of input tokens is encoded into the representations, we consider functions whose value is determined by the representation (after a perturbation of noise) \citep{srinivas2009gaussian}. 

Specifically, given representation $\mathbf{z}$, we assume the observed value is defined as follows:
\begin{equation*}
y = f(\mathbf{z}) + \epsilon,
\end{equation*}
where $\mathbf{z}$ is the representation and $\epsilon$ a noise term.

Consider a \emph{sequence} of representations $\mathbf{z}_1, \cdots, \mathbf{z}_T  $ generated by the input token sequence, we model $f$ as a sample from Gaussian process (GP) following prior work \citep{srinivas2009gaussian}. The assumption of GP will help us analyze the influence of observed representations, which we will explain as follows.

By noticing that the mean can be subtracted and the covariance function should reflect the relationship between embeddings, we assume the prior of $f$ is $\text{GP}(0, k(\mathbf{z}, \mathbf{z}'))$, where $k$ is a kernel function. Assume $y_t = f(\mathbf{z}_t) + \epsilon_t$, where $\epsilon_t \sim \mathcal{N}(0, \sigma^2)$ is a Gaussian noise.

When we have observed $\mathbf{y}_T =[y_1, \cdots, y_{T}]$, we are interested in the posterior of $f$, which reflects the observed sequence's influence. Then the posterior over $f$ will also be a GP, with mean function $\mu_T(\boldsymbol{z})$, covariance function $k_T\left(\boldsymbol{z}, \boldsymbol{z}^{\prime}\right)$, variance $\sigma^2_T(\mathbf{z}) = k_T(\mathbf{z}, \mathbf{z})$:

\begin{align}
\mu_T(\boldsymbol{z}) & ={k}_T(\boldsymbol{z})^T\left(\mathbf{K}_T+\sigma^2 \mathbf{I}_T\right)^{-1} \boldsymbol{y}_T, \nonumber \\
k_T\left(\boldsymbol{z}, \boldsymbol{z}^{\prime}\right) & =k\left(\boldsymbol{z}, \boldsymbol{z}^{\prime}\right)-{k}_T(\boldsymbol{z})^T\left(\mathbf{K}_T+\sigma^2 \mathbf{I}_T\right)^{-1} {k}_T\left(\boldsymbol{z}^{\prime}\right),
\label{posterior variance}
\end{align}

where ${k}_T(\boldsymbol{z})=\left[k\left(\boldsymbol{z}_1, \boldsymbol{z}\right) \ldots k\left(\boldsymbol{z}_T, \boldsymbol{z}\right)\right]^T$ and $\mathbf{K}_T$ is kernel matrix $\left[k\left(\boldsymbol{z}_i, \boldsymbol{z}_j\right)\right]_{i,j}$.

From an information-theoretic view, the first $T$ may have the following number of mutual information shared with $f$ which is called \emph{information gain}:
$$
\operatorname{I}( \mathbf{y}_T; f) = \frac{1}{2} \log \det (\mathbf{I}_T + \sigma^{-2} \mathbf{K}_T).
$$

\begin{theorem}
Assume $k(\mathbf{z}, \mathbf{z}) = 1$, then we can see that $$\operatorname{I}( \mathbf{y}_{T+1}; f) - \operatorname{I}( \mathbf{y}_T; f) = \frac{\log ( (1 + \sigma^{-2} -\sigma^2 )+\sigma^2 \sigma^2_T(\mathbf{z}_{T+1}))}{2}.$$
\end{theorem}

\begin{proof}
Note \begin{equation*}
\mathbf{I}_{T+1} + \sigma^{-2} \mathbf{K}_{T+1} = \begin{bmatrix}
\mathbf{I}_T + \sigma^{-2} \mathbf{K}_T & k_T(\mathbf{z}_{T+1})  \\
k_T(\mathbf{z}_{T+1})^T &  1 + \sigma^{-2}
\end{bmatrix}.
\end{equation*}

Denote $\mathbf{s} = (\mathbf{I}_T + \sigma^{-2} \mathbf{K}_T)^{-1} k_T(\mathbf{z}_{T+1})$, from Gaussian elimination, we know that \begin{equation*}
\det (\mathbf{I}_{T+1} + \sigma^{-2} \mathbf{K}_{T+1} )= \det( \begin{bmatrix}
\mathbf{I}_T + \sigma^{-2} \mathbf{K}_T & 0  \\
k_T(\mathbf{z}_{T+1})^T &  1 + \sigma^{-2} - k_T(\mathbf{z}_{T+1})^T\mathbf{s}
\end{bmatrix} ).
\end{equation*}

Therefore, the conclusion follows by noticing $\sigma^2_T(\mathbf{z}) = k_T(\mathbf{z}, \mathbf{z})$ and equation (\ref{posterior variance}).
\end{proof}

Therefore the difference of information gain is closely linked with the token's posterior uncertainty (variance). We will then show that when the kernel is the inner product kernel, this uncertainty is mainly determined by the linear approximation.

As We are interested in the linear structure of the representation sequence. Note the representations of $\mathbf{z}_1 \cdots \mathbf{z}_T$ may be collinear, we consider the ridge regression of $\mathbf{z}_{T+1}$ on the previous representations $\mathbf{Z}_T = [\mathbf{z}_1 \cdots \mathbf{z}_T]$ as follows:
\begin{equation*}
\min_{\beta} \|\mathbf{z}_{T+1} - \mathbf{Z}_T \beta \|^2_2 + {\sigma^2}  \|\beta\|^2_2.
\end{equation*}

Then it is easy to see that $\beta_{\text{Ridge}} = (\mathbf{Z}^T_T\mathbf{Z}_T+\sigma^2 \mathbf{I}_T)^{-1}\mathbf{Z}^T_T\mathbf{z}_{T+1}$. 

Assume $k(\mathbf{z}, \mathbf{z}') = \mathbf{z}^T\mathbf{z}'$. Denote $\hat{\mathbf{z}}_{T+1} = \mathbf{Z}_T\beta_{\text{Ridge}}$, then equation (\ref{posterior variance}) can be rewritten as 
\begin{equation*}
\sigma^2_T(\mathbf{z}_{T+1}) = k_T(\mathbf{z}_{T+1}, \mathbf{z}_{T+1}) = \langle \mathbf{z}_{T+1}, \mathbf{z}_{T+1}- \hat{\mathbf{z}}_{T+1} \rangle.
\end{equation*}

In the following, we will always consider $k(\mathbf{z}, \mathbf{z}') = \mathbf{z}^T\mathbf{z}'$, where $\mathbf{z}$ is $l_2$ normalized.

Thus the difference of information gain will help us understand the next token's linear dependency on the previous tokens.

\begin{definition}[Stable rank (Numerical rank)]
Given a matrix $\mathbf{A}$, the stable rank of $\mathbf{A}$ is defined as:
\begin{equation*}
\text{srank}(\mathbf{A}) = \frac{\| \mathbf{A}\|^{2}_F}{\| \mathbf{A} \|^2_2}.
\end{equation*}
\end{definition}

Note $\operatorname{I}( \mathbf{y}_T; f) = \sum^{\text{rank}(\mathbf{K}_T)}_{i=1} \log(1 +\sigma^{-2} \lambda_i) \leq \text{rank}(\mathbf{K}_T) \log(1 +\sigma^{-2} \lambda_1) = \text{rank}(\mathbf{Z}_T) \log(1 +\sigma^{-2} \frac{T}{\text{srank}(\mathbf{Z}_T)}) \sim \text{srank}(\mathbf{Z}_T) \log(1 +\sigma^{-2} \frac{T}{\text{srank}(\mathbf{Z}_T)})$, where $\lambda_1$ is the largest eigenvalue of $\mathbf{K}_T$.

As $x \log(1 +\sigma^{-2} \frac{T}{x}) $ is an increasing function of $x$. By the above bound, $\operatorname{I}( \mathbf{y}_T; f)$ relates to the rank and therefore relates to the hidden dimension value, so we should take the hidden dimension of the model into account. Interestingly, by the fact that $\operatorname{I}( \mathbf{y}_T; f) \leq d \log (1 + \sigma^{-2} T)$, then normalization by $d$ will give a fair comparison among different size models.

We evaluate the information gain and information gain difference on Pythia models \citep{biderman2023pythia} on hh-rlhf dataset \citep{bai2022training} with $\sigma=0.01$.

\begin{figure}[htb]
\centering
\begin{subfigure}[b]{0.46\columnwidth}
\includegraphics[width=\linewidth]{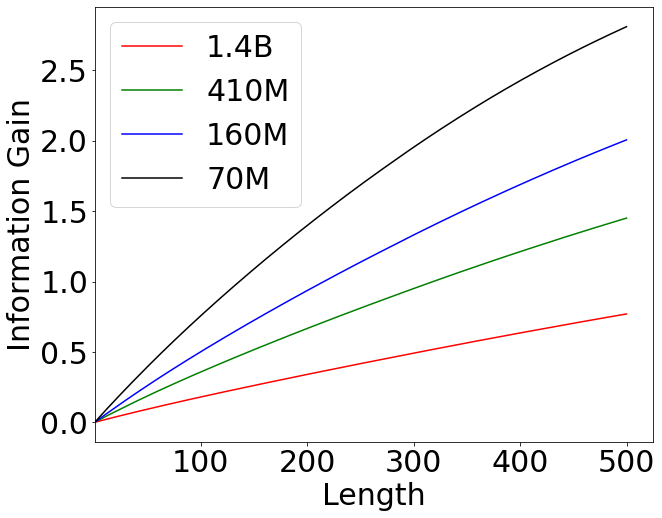}
\caption{The normalized information gain.}
\end{subfigure}
\begin{subfigure}[b]{0.49\columnwidth}
\includegraphics[width=\linewidth]{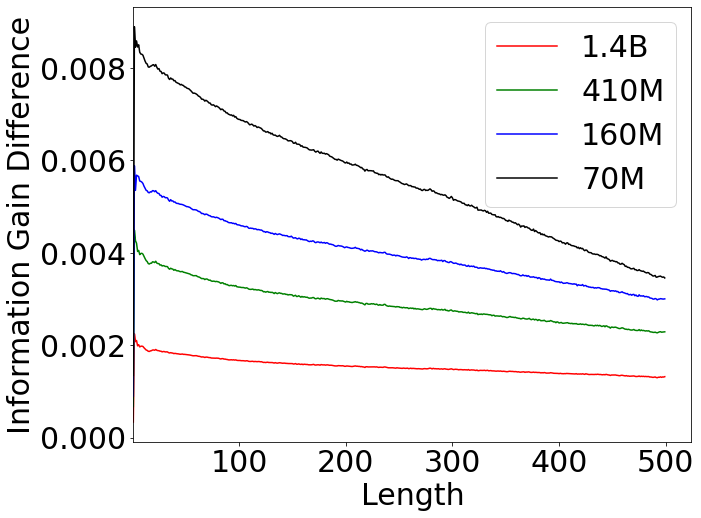}
\caption{The difference of normalized information gain.}
\end{subfigure}
\caption{Quantities related to the information gain.}
\label{fig: info gain}
\end{figure}

Figure \ref{fig: info gain} presents the findings, which demonstrate that as the input length increases, the normalized information gain also increases. However, when the models are larger, the increase is more moderate. Additionally, the difference of information gain decreases as the input length increases, but the decrease is less pronounced when the models are larger.

\subsection{Attention and Lasso}

Another important quantity in GPT models is the attention weight, which plays a central role in the attention module.

Assume the value matrix embedding of input tokens be $\mathbf{v}_1, \cdots, \mathbf{v}_{T+1}$, then by the auto-regressive nature we have
\begin{equation} \label{att}
\mathbf{z}_{T+1} = \lambda^{T+1}_{T+1} \mathbf{v}_{T+1} + \sum_{j<T+1} \lambda^{T+1}_{j} \mathbf{v}_{j},
\end{equation}
where $\lambda^{T+1}_j$ is the attention weight of $T+1$-th token to the $j$-th token.

Note that
\begin{equation} \label{recursive att}
\mathbf{v}_{i} = \frac{1}{\lambda^{i}_{i}} \mathbf{z}_{i} - \sum_{j<i} \frac{\lambda^{i}_{j}}{\lambda^{i}_{i}} \mathbf{v}_{j},
\end{equation}
where $\lambda^{i}_j$ is the attention weight of $i$-th token to the $j$-th token.

As the attention weights are usually very sparse, substituting equation (\ref{recursive att}) into equation (\ref{att}), one may have
\begin{align*}
\mathbf{z}_{T+1} =& \lambda^{T+1}_{T+1} \mathbf{v}_{T+1} + \sum_{\lambda^{T+1}_{j} \neq 0} \lambda^{T+1}_{j} (\frac{1}{\lambda^{j}_{j}} \mathbf{z}_{j} - \sum_{k<j} \frac{\lambda^{j}_{k}}{\lambda^{j}_{j}} \mathbf{v}_{k}) \\     
 =& \lambda^{T+1}_{T+1} \mathbf{v}_{T+1} + \sum_{\lambda^{T+1}_{j} \neq 0}  \frac{\lambda^{T+1}_{j}}{\lambda^{j}_{j}} \mathbf{z}_{j} - \sum_{\lambda^{T+1}_{j} \neq 0} \lambda^{T+1}_{j} \sum_{\lambda^{j}_{k} \neq 0} \frac{\lambda^{j}_{k}}{\lambda^{j}_{j}} \mathbf{v}_{k}.
\end{align*}

By noticing that $\frac{\lambda^{j}_{k}}{\lambda^{j}_{j}}$ is the exponential of logit difference which makes $\frac{\lambda^{j}_{k}}{\lambda^{j}_{j}}$ usually small and $\sum_j \lambda^{T+1}_{j} = 1$ will make each $\lambda^{T+1}_{j}$ small, then we have the following approximate:
\begin{align}  \label{lasso att}
\mathbf{z}_{T+1} \approx \lambda^{T+1}_{T+1} \mathbf{v}_{T+1} + \sum_{\lambda^{T+1}_{j} \neq 0}  \frac{\lambda^{T+1}_{j}}{\lambda^{j}_{j}} \mathbf{z}_{j} .
\end{align}

Equation (\ref{lasso att}) motivates us to consider another regularized regression technique: Lasso. Consider the Lasso regression of $\mathbf{z}_{T+1}$ on the previous representations $\mathbf{Z}_T = [\mathbf{z}_1 \cdots \mathbf{z}_T]$ as follows:
\begin{equation*}
\min_{\beta} \|\mathbf{z}_{T+1} - \mathbf{Z}_T \beta \|^2_2 + {\lambda}  \|\beta\|_1.
\end{equation*}

We consider performing an empirical investigation to see the performance of attention weight and Lasso in finding meaningful tokens. Note there are multiple attention heads, we average them as the attention weight. We take $\lambda = 0.0001$ in Lasso. For all cases, we pick weights bigger than $0.01$. We conduct experiments on Pythia 1.4B.

We choose to consider the following sentence, which has a very clear structure: 

``Is London the capital of UK? Yes. Is Moscow the capital of Russia? Yes. Is Paris the capital of France? Yes. Is Beijing the capital of China? No''

For token [No]: The largest attention weights pick (apart from itself): [Is], [.] and [?]. The Lasso regression picks: [Yes], [Yes].

For token [?]: The largest attention weights pick (apart from itself): [Is], [.], [?], [.], [.], [of], [?]. The Lasso regression picks: [?], [?].

For token [China]: The largest attention weights pick (apart from itself): [Is], [of], [capital], [.], [Beijing]. The Lasso regression picks: [Russia], [France], [Beijing], [of], [capital].

The model will have an MLP layer after the attention module. Combined with the fact that Lasso may usually pick more meaningful tokens than raw attention, this shows that MLP also plays an important role in shaping good representations.

\subsection{Does a token embedding contain all the information from its preceding context?}

As LLM's auto-regressive nature, it seems that all the information in the contexts is encoded into the new token embeddings. Therefore, we want to design sentences like ``context + [token].''. For ease of visualization, we choose two different contexts (and an additional empty context) and varying different [token]s.

The contexts we choose are ``He is the king of'' and ``She is the queen of'' and tokens are chosen from country names (country names can be found in Appendix \ref{prompt}).

We plot the visualization result of in Figure \ref{fig:visual}. The distance is the $l_2$ norm and we use the mean vector of all the token embeddings as the baseline (contains all token's information). Note, compared to the visualization based only on (last) [token], the mean vectors captured much more meaningful semantics as it distinguishes that one context is mainly about man (king) and another about woman (queen). We also plot the principal component analysis (PCA) and find it also visualizes very well.

\begin{figure*}[htb]
\centering
\begin{subfigure}[b]{0.49\columnwidth}
\includegraphics[width=\linewidth]{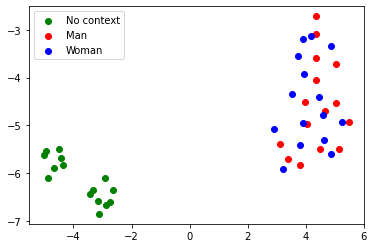}
\caption{One vector.}
\end{subfigure}
\begin{subfigure}[b]{0.49\columnwidth}
\includegraphics[width=\linewidth]{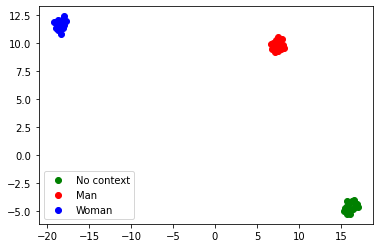}
\caption{Mean vector.}
\end{subfigure}

\begin{subfigure}[b]{0.6\columnwidth}
\includegraphics[width=\linewidth]{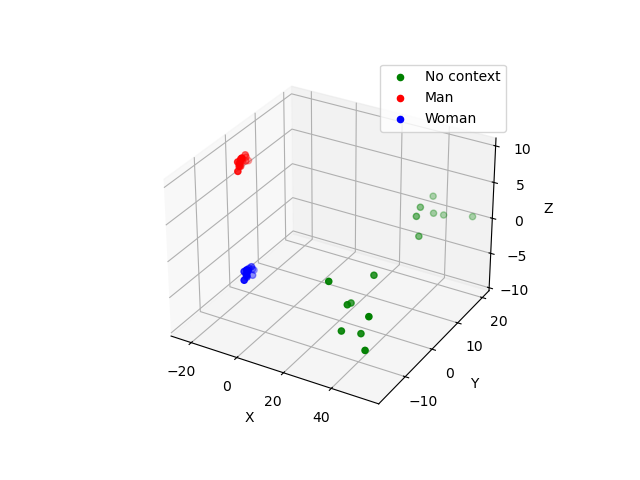}
\caption{PCA of mean vectors.}
\end{subfigure}
\caption{Visualization by UMAP and PCA.}
\label{fig:visual}
\end{figure*}

As we find information is distributed among tokens. Note mean vector can be seen as capturing the first-order information as it resembles sample means, what about second-order relationships encoded in the covariance matrix? We need to define distance functions between covariance matrices. Recall the Logdet distance \citep{sra2016positive} is given as follows:
\begin{equation*}
d_{\text{Logdet}}(\mathbf{A}, \mathbf{B}) = \sqrt{\log \operatorname{det} (\frac{\mathbf{A} + \mathbf{B}}{2}) - \frac{\log \operatorname{det}(\mathbf{A}) + \log \operatorname{det}(\mathbf{B}) }{2}}.    
\end{equation*}

Motivated by the Logdet distance and the Logdet approximation of entropy (definition \ref{log det entropy}), we will define the following distance which resembles the (Logdet) Jensen-Shannon distance: 

\begin{equation*}
d_{\text{JS}}(\mathbf{A}, \mathbf{B}) = \sqrt{\log \operatorname{det} (\mathbf{I}_d 
 + \gamma \frac{\mathbf{A} + \mathbf{B}}{2}) - \frac{\log \operatorname{det}(\mathbf{I}_d 
 + \gamma \mathbf{A}) + \log \operatorname{det}(\mathbf{I}_d 
 + \gamma \mathbf{B}) }{2}},    
\end{equation*}
where $\gamma$ is a hyper-parameter and in experiments we set $\gamma=100$.

\begin{figure}[t] 
\centering 
\includegraphics[width=0.4\columnwidth]{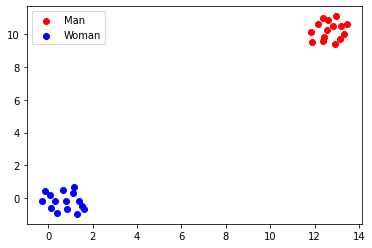}
\caption{Sentence distance using JS distance.}
\label{fig: JS}
\end{figure}

From Figure \ref{fig: JS}, the ``second-order'' relationship encoded in LLM embeddings is also useful as the visualization is quite good. For ablations of other distances, please refer to Appendix \ref{ablation of other metric}.

\section{Related Work}

\paragraph{Information theory.}
Information theory has been shown to be a powerful tool for understanding supervised learning and visual self-supervised learning method \citep{tishby2015deep, bach2022information, tan2023information, zhang2023kernel}. Recently, \citep{wei2024large} extends to use the (normalized) matrix entropy to evaluate the concept of ``compression'' in large language models, taking a step towards understanding the inner workings of large language models. There have also been works that explore the relationship between large language models and the concept of lossless compression in information theory \citep{deletang2023language, valmeekam2023llmzip}.

\paragraph{Scaling law.}

The performance of models in LLMs exhibits a fascinating phenomenon known as the scaling law \citep{kaplan2020scaling, hoffmann2022training}. This law reveals a predictable power-law relationship between parameters/training flops/data size and performance. Although scaling laws have been empirically observed in various fields, their theoretical understanding remains scarce \citep{bahri2021explaining, michaud2023quantization}.

\section{Conclusion}

In this paper, we investigate the information encoded in the large language model (LLM) embeddings. We first simulate the representation entropy and find it follows a power law relationship with model sizes. We then provide a theory based on (conditional) entropy that can explain the scaling law of entropy. As modern LLM has an auto-regressive structure, we investigate the last token's relationship to the previously generated tokens using tools from information theory, Gaussian process, and regression. We find that the information gain of the new token is theoretically connected with ridge regression. Moreover, motivated by the close relationship between the Lasso regression and attention mechanism, we find that the Lasso regression can pick meaningful tokens and the tokens picked by Lasso are sometimes even more intuitive than attention weights, showcasing the important role of the MLP layer in LLM. Finally, we discuss metrics that can capture the sentence-level distance by incorporating information theory and LLM representations.

\bibliography{reference}
\bibliographystyle{iclr2024_conference}

%%%%%%%%%%%%%%%%%%%%%%%%%%%%%%%%%%%%%%%%%%%%%%%%%%%%%%%%%%%%%%%%%%%%%%%%%%%%%%%
% APPENDIX
%%%%%%%%%%%%%%%%%%%%%%%%%%%%%%%%%%%%%%%%%%%%%%%%%%%%%%%%%%%%%%%%%%%%%%%%%%%%%%%
\clearpage
\appendix
\section{Ablation study}  \label{ablation of other metric}

\begin{figure*}
\centering
\begin{subfigure}[b]{0.24\columnwidth}
\includegraphics[width=\linewidth]{image/sentence_dis.png}
\caption{Logdet based.}
\end{subfigure}
\begin{subfigure}[b]{0.24\columnwidth}
\includegraphics[width=\linewidth]{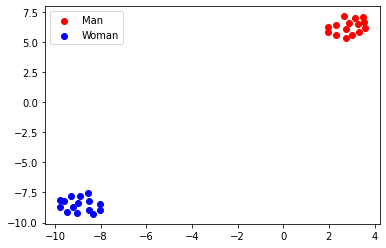}
\caption{Riemann based.}
\end{subfigure}
\begin{subfigure}[b]{0.24\columnwidth}
\includegraphics[width=\linewidth]{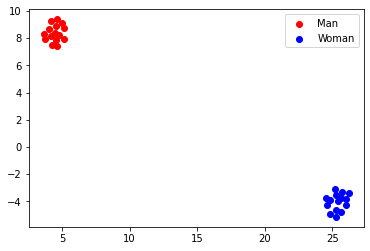}
\caption{LogE based.}
\end{subfigure}
\begin{subfigure}[b]{0.24\columnwidth}
\includegraphics[ width=\linewidth]{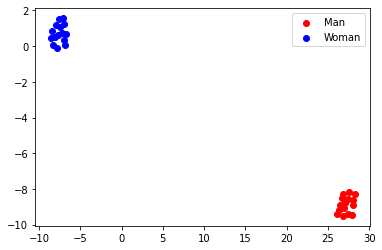}
\caption{Frobenius norm based.}
\end{subfigure}

\caption{Different distance functions.}
\label{fig: different dis}
\end{figure*}

We also consider other distances \citep{sra2016positive} between positive definite matrices as follows:

For example Riemannian distance:
\begin{equation*}
d_{\text{Riemann}}(\mathbf{A}, \mathbf{B}) = \| \log (\mathbf{B}^{-\frac{1}{2}} \mathbf{A} \mathbf{B}^{-\frac{1}{2}})  \|_F.    
\end{equation*}

Also the log-Euclidean distance:
\begin{equation*}
d_{\text{LogE}}(\mathbf{A}, \mathbf{B}) = \| \log \mathbf{A} -  \log \mathbf{B} \|_F.    
\end{equation*}

In Figure \ref{fig: different dis}, we find that all distance functions can capture the semantic of sentences well, further indicating that information is distributed among tokens and this conclusion does not depend on the choice of the distance function.

Interestingly, when $\gamma$ is small, the JS distance has a close relationship with the Frobenius norm.

\begin{theorem}
Assume $\gamma < \frac{1}{\max \{ \| \mathbf{A} \|, \| \mathbf{B} \|  \} }$, then $d_{\text{JS}}(\mathbf{A}, \mathbf{B}) = \sqrt{ \frac{\gamma^2}{8} \| \mathbf{A} - \mathbf{B} \|^2_F + O(\gamma^3)} $.  
\end{theorem}

\begin{proof}
Note using Taylor expansion of matrix logarithm, we have 
\begin{align*}
d^2_{\text{JS}}(\mathbf{A}, \mathbf{B}) =& \log \operatorname{det} (\mathbf{I}_d 
 + \gamma \frac{\mathbf{A} + \mathbf{B}}{2}) - \frac{\log \operatorname{det}(\mathbf{I}_d 
 + \gamma \mathbf{A}) + \log \operatorname{det}(\mathbf{I}_d 
 + \gamma \mathbf{B}) }{2} \\
 =& \operatorname{tr}(\log(\mathbf{I}_d 
 + \gamma \frac{\mathbf{A} + \mathbf{B}}{2})) - \frac{\operatorname{tr}(\log(\mathbf{I}_d 
 + \gamma \mathbf{A} )) + \operatorname{tr}(\log(\mathbf{I}_d 
 + \gamma \mathbf{B} ))}{2} \\
 =& \operatorname{tr}(\gamma \frac{\mathbf{A} + \mathbf{B}}{2} -  \frac{1}{2} (\gamma \frac{\mathbf{A} + \mathbf{B}}{2})^2 + O(\gamma^3)) - \frac{\operatorname{tr}(\gamma(\mathbf{A} + \mathbf{B}) - \frac{\gamma^2}{2} (\mathbf{A}^2 + \mathbf{B}^2) + O(\gamma^3) ) }{2} \\
 =&  \frac{\gamma^2}{8} \| \mathbf{A} - \mathbf{B} \|^2_F + O(\gamma^3).
\end{align*}
\end{proof}

\section{Country names in the prompt} \label{prompt}

The names used are: "Australia.", "Austria.",  "Brazil.",
         "Cuba.", "Denmark.", "Egypt.",
         "France.", "Germany.", "Hungary.",
         "Iceland.", "India.", "Japan.",
         "Norway.", "Poland.", "Spain.".

\end{document}